\documentclass[sigconf]{acmart}
\AtBeginDocument{%
  }

\setcopyright{acmlicensed}
\copyrightyear{2018}
\acmYear{2018}
\acmDOI{XXXXXXX.XXXXXXX}
\acmConference[Conference acronym 'XX]{Make sure to enter the correct
  conference title from your rights confirmation email}{June 03--05,
  2018}{Woodstock, NY}
\acmISBN{978-1-4503-XXXX-X/2018/06}

\usepackage{mathtools}

\usepackage{hyperref}
\usepackage{url}

\usepackage{natbib}
\usepackage{graphicx}
\usepackage[utf8]{inputenc} 
\usepackage[T1]{fontenc}    
\usepackage{hyperref}       
\usepackage{url}            
\usepackage{booktabs}       
\usepackage{amsfonts}       
\usepackage{nicefrac}       
\usepackage{microtype}      
\usepackage{xcolor}         
\usepackage{amsmath}
\usepackage{enumitem}
\usepackage{multirow}
\usepackage{threeparttable}
\usepackage{subfig}
\usepackage{circledsteps}
\usepackage{amsthm}
\usepackage{wrapfig2}

\theoremstyle{plain}
\newtheorem{theorem}{Theorem}
\usepackage[textsize=tiny]{todonotes}
\usepackage[outercaption]{sidecap}
\usepackage{multirow}
\usepackage{xcolor}
\usepackage{graphicx} 

\begin{document}




\newtheorem{algo}{Algorithm}

\newcommand{\overbar}[1]{\mkern 1.5mu\overline{\mkern-1.5mu#1\mkern-1.5mu}\mkern 1.5mu}

\newcommand{\outline}[1]{{\color{brown}#1}}
\newcommand{\rev}[1]{{\color{blue}#1}}
\newcommand{\remove}[1]{{\sout{#1}}}
\newcommand{\cut}[1]{{}}
\newcommand{\red}[1]{{\color{red}#1}} 
\newcommand{\blue}[1]{{\color{blue}#1}} 

\newcommand{\method}{\text{E2E-GRec }}
\newcommand{\tA}{{\tilde{\vA}}}
\newcommand{\tD}{{\tilde{\vD}}}
\newcommand{\tL}{{\tilde{\vL}}}
\newcommand{\hA}{{\hat{\vA}}}
\newcommand{\hD}{{\hat{\vD}}}
\newcommand{\hL}{{\hat{\vL}}}
\newcommand{\tDelta}{{\tilde{\Delta}}}
\newcommand{\Xin}{{\vX_{\text{in}}}}
\newcommand{\Xinit}{{\vX_{\text{init}}}}
\newcommand{\tin}{{\text{in}}}
\newcommand{\tinit}{{\text{init}}}


\newcommand{\va}{{\mathbf{a}}}
\newcommand{\vb}{{\mathbf{b}}}
\newcommand{\vc}{{\mathbf{c}}}
\newcommand{\vd}{{\mathbf{d}}}
\newcommand{\ve}{{\mathbf{e}}}
\newcommand{\vf}{{\mathbf{f}}}
\newcommand{\vg}{{\mathbf{g}}}
\newcommand{\vh}{{\mathbf{h}}}
\newcommand{\vi}{{\mathbf{i}}}
\newcommand{\vj}{{\mathbf{j}}}
\newcommand{\vk}{{\mathbf{k}}}
\newcommand{\vl}{{\mathbf{l}}}
\newcommand{\vm}{{\mathbf{m}}}
\newcommand{\vn}{{\mathbf{n}}}
\newcommand{\vo}{{\mathbf{o}}}
\newcommand{\vp}{{\mathbf{p}}}
\newcommand{\vq}{{\mathbf{q}}}
\newcommand{\vr}{{\mathbf{r}}}
\newcommand{\vs}{{\mathbf{s}}}
\newcommand{\vt}{{\mathbf{t}}}
\newcommand{\vu}{{\mathbf{u}}}
\newcommand{\vw}{{\mathbf{w}}}
\newcommand{\vx}{{\mathbf{x}}}
\newcommand{\vy}{{\mathbf{y}}}
\newcommand{\vz}{{\mathbf{z}}}

\newcommand{\vA}{{\mathbf{A}}}
\newcommand{\vB}{{\mathbf{B}}}
\newcommand{\vC}{{\mathbf{C}}}
\newcommand{\vD}{{\mathbf{D}}}
\newcommand{\vE}{{\mathbf{E}}}
\newcommand{\vF}{{\mathbf{F}}}
\newcommand{\vG}{{\mathbf{G}}}
\newcommand{\vH}{{\mathbf{H}}}
\newcommand{\vI}{{\mathbf{I}}}
\newcommand{\vJ}{{\mathbf{J}}}
\newcommand{\vK}{{\mathbf{K}}}
\newcommand{\vL}{{\mathbf{L}}}
\newcommand{\vM}{{\mathbf{M}}}
\newcommand{\vN}{{\mathbf{N}}}
\newcommand{\vO}{{\mathbf{O}}}
\newcommand{\vP}{{\mathbf{P}}}
\newcommand{\vQ}{{\mathbf{Q}}}
\newcommand{\vR}{{\mathbf{R}}}
\newcommand{\vS}{{\mathbf{S}}}
\newcommand{\vT}{{\mathbf{T}}}
\newcommand{\vU}{{\mathbf{U}}}
\newcommand{\vV}{{\mathbf{V}}}
\newcommand{\vW}{{\mathbf{W}}}
\newcommand{\vX}{{\mathbf{X}}}
\newcommand{\vY}{{\mathbf{Y}}}
\newcommand{\vZ}{{\mathbf{Z}}}

\newcommand{\cA}{{\mathcal{A}}}
\newcommand{\cB}{{\mathcal{B}}}
\newcommand{\cC}{{\mathcal{C}}}
\newcommand{\cD}{{\mathcal{D}}}
\newcommand{\cE}{{\mathcal{E}}}
\newcommand{\cF}{{\mathcal{F}}}
\newcommand{\cG}{{\mathcal{G}}}
\newcommand{\cH}{{\mathcal{H}}}
\newcommand{\cI}{{\mathcal{I}}}
\newcommand{\cJ}{{\mathcal{J}}}
\newcommand{\cK}{{\mathcal{K}}}
\newcommand{\cL}{{\mathcal{L}}}
\newcommand{\cM}{{\mathcal{M}}}
\newcommand{\cN}{{\mathcal{N}}}
\newcommand{\cO}{{\mathcal{O}}}
\newcommand{\cP}{{\mathcal{P}}}
\newcommand{\cQ}{{\mathcal{Q}}}
\newcommand{\cR}{{\mathcal{R}}}
\newcommand{\cS}{{\mathcal{S}}}
\newcommand{\cT}{{\mathcal{T}}}
\newcommand{\cU}{{\mathcal{U}}}
\newcommand{\cV}{{\mathcal{V}}}
\newcommand{\cW}{{\mathcal{W}}}
\newcommand{\cX}{{\mathcal{X}}}
\newcommand{\cY}{{\mathcal{Y}}}
\newcommand{\cZ}{{\mathcal{Z}}}

\newcommand{\ri}{{\mathrm{i}}}
\newcommand{\rr}{{\mathrm{r}}}

\newcommand{\EE}{{\mathbb{E}}}


\newcommand{\RR}{\mathbb{R}}
\newcommand{\CC}{\mathbb{C}}
\newcommand{\ZZ}{\mathbb{Z}}
\renewcommand{\SS}{{\mathbb{S}}}
\newcommand{\SSp}{\mathbb{S}_{+}}
\newcommand{\SSpp}{\mathbb{S}_{++}}
\newcommand{\sign}{\mathrm{sign}}
\newcommand{\vzero}{\mathbf{0}}
\newcommand{\vone}{{\mathbf{1}}}

\newcommand{\st}{{\text{s.t.}}} 
\newcommand{\St}{{\mathrm{subject~to}}} 
\newcommand{\op}{{\mathrm{op}}} 
\newcommand{\opt}{{\mathrm{opt}}} 
\newcommand{\Prob}{{\mathrm{Prob}}} 
\newcommand{\Diag}{{\mathrm{Diag}}} 
\newcommand{\dom}{{\mathrm{dom}}} 
\newcommand{\range}{{\mathbf{range}}} 
\newcommand{\tr}{{\mathrm{tr}}} 
\newcommand{\TV}{{\mathrm{TV}}} 
\newcommand{\Proj}{{\mathrm{Proj}}}
\newcommand{\prj}{{\mathrm{prj}}}
\newcommand{\prox}{\mathbf{prox}}
\newcommand{\refl}{\mathbf{refl}}
\newcommand{\reflh}{\refl^{\bH}}
\newcommand{\proxh}{\prox^{\bH}}
\newcommand{\minimize}{\text{minimize}}
\newcommand{\bgamma}{\boldsymbol{\gamma}}
\newcommand{\bsigma}{\boldsymbol{\sigma}}
\newcommand{\bomega}{\boldsymbol{\omega}}
\newcommand{\blambda}{\boldsymbol{\lambda}}
\newcommand{\bH}{\vH}
\newcommand{\bbH}{\mathbb{H}}
\newcommand{\bB}{\boldsymbol{\cB}}
\newcommand{\Tau}{\mathrm{T}}
\newcommand{\tnabla}{\widetilde{\nabla}}
\newcommand{\TDRS}{T_{\mathrm{DRS}}}
\newcommand{\TPRS}{T_{\mathrm{PRS}}}
\newcommand{\TFBS}{T_{\mathrm{FBS}}}
\newcommand{\best}{\mathrm{best}}
\newcommand{\kbest}{k_{\best}}
\newcommand{\diff}{\mathrm{diff}}
\newcommand{\barx}{\bar{x}}
\newcommand{\xgbar}{\bar{x}_g}
\newcommand{\xfbar}{\bar{x}_f}
\newcommand{\hatxi}{\hat{\xi}}
\newcommand{\xg}{x_g}
\newcommand{\xf}{x_f}
\newcommand{\du}{\mathrm{d}u}
\newcommand{\dy}{\mathrm{d}y}
\newcommand{\kconvergence}{\stackrel{k \rightarrow \infty}{\rightarrow }}
\newcommand{\Null}{\mathbf{Null}}
\newcommand{\Span}{\mathbf{span}}

\makeatletter
\let\@@span\span
\def\sp@n{\@@span\omit\advance\@multicnt\m@ne}
\makeatother

\DeclarePairedDelimiter{\dotpb}{\langle}{\rangle_{\bH}}
\DeclarePairedDelimiter{\dotpv}{\langle}{\rangle_{\vH}}
\DeclarePairedDelimiter{\dotp}{\langle}{\rangle}


\newcommand{\MyFigure}[1]{../fig/#1}

\newcommand{\bc}{\begin{center}}
\newcommand{\ec}{\end{center}}

\newcommand{\bdm}{\begin{displaymath}}
\newcommand{\edm}{\end{displaymath}}

\newcommand{\beq}{\begin{equation}}
\newcommand{\eeq}{\end{equation}}

\newcommand{\bfl}{\begin{flushleft}}
\newcommand{\efl}{\end{flushleft}}

\newcommand{\bt}{\begin{tabbing}}
\newcommand{\et}{\end{tabbing}}

\newcommand{\beqn}{\begin{align}}
\newcommand{\eeqn}{\end{align}}

\newcommand{\beqs}{\begin{align*}} 
\newcommand{\eeqs}{\end{align*}}  


\newtheorem{condition}{Condition}
\newtheorem{rul}{Rule}

\newcommand{\Ker}{\mathbf{Ker}}
\newcommand{\Range}{\mathbf{Range}} 
\title{E2E-GRec: An End-to-End Joint Training Framework for Graph Neural Networks and Recommender Systems}

\author{Rui Xue}
\email{rxue@ncsu.edu}
\affiliation{%
  \institution{North Carolina State University}
  \city{Raleigh}
  \authornote{This work was done while the author was an intern at TikTok Inc.}
  \country{US}
}

\author{Shichao Zhu}
\affiliation{%
  \institution{TikTok Inc.}
  \city{San Jose}
  \country{US}}

\author{Liang Qin}
\affiliation{%
  \institution{TikTok Inc.}
  \city{Chengdu}
  \country{CN}}

\author{Tianfu Wu}
\affiliation{%
  \institution{North Carolina State University}
  \city{Raleigh}
  \country{US}
}



\begin{abstract}
Graph Neural Networks (GNNs) have emerged as powerful tools for modeling graph-structured data and have been widely used in recommender systems, such as for capturing complex user–item and item–item relations. However, most industrial deployments adopt a two-stage pipeline: GNNs are first pre-trained offline to generate node embeddings, which are then used as static features for downstream recommender systems. This decoupled paradigm leads to two key limitations: (1) high computational overhead, since large-scale GNN inference must be repeatedly executed to refresh embeddings; and (2) lack of joint optimization, as the gradient from the recommender system cannot directly influence the GNN learning process, causing the GNN to be suboptimally informative for the recommendation task. In this paper, we propose \textbf{E2E-GRec}, a novel \textbf{end-to-end} training framework that unifies GNN training with the recommender system. Our framework is characterized by three key components:
(i) efficient subgraph sampling from a large-scale cross-domain heterogeneous graph to ensure training scalability and efficiency;
(ii) a Graph Feature Auto-Encoder (GFAE) serving as an auxiliary self-supervised task to guide the GNN to learn structurally meaningful embeddings; and
(iii) a two-level feature fusion mechanism combined with Gradnorm-based dynamic loss balancing, which stabilizes graph-aware multi-task end-to-end training.
Extensive offline evaluations, online A/B tests (e.g., a +0.133\% relative improvement in stay duration, a 0.3171\% reduction in the average
number of videos a user skips) on large-scale production data, together with theoretical analysis, demonstrate that E2E-GRec consistently surpasses traditional approaches, yielding significant gains across multiple recommendation metrics.
\end{abstract}

\maketitle

\section{Introduction}
\label{sec:intro}

Recommender systems have become an indispensable component of modern digital platforms, driving personalized content delivery across e-commerce, social media, and video-streaming services \citep{ricci2010introduction, aggarwal2016introduction,zhang2019deep}. As users are increasingly overwhelmed by massive and continuously expanding content catalogs, the ability to accurately model user preferences directly impacts user engagement, retention, and overall platform revenue \citep{covington2016deep,zhou2018deep}.

Graph Neural Networks (GNNs) have emerged as a powerful paradigm for capturing the complex relational structures inherent in graph-structured data \cite{kipf2016semi,gasteiger2018combining,velivckovic2017graph,hamilton2017inductive,wu2019simplifying,xue2023lazygnn,xue2024haste,xue2025visagnn,xue2025h} and have achieved remarkable success across multiple domains, including biological networks, transportation infrastructures, and recommender systems~\citep{fout2017protein,tang2020knowing,sankar2021graph,wu2022graph,zhang2024linear,xue2023efficient, xue2023large}.
By iteratively aggregating information from neighboring nodes, GNNs can effectively model various relations to learn expressive, high-order representations that capture both local neighborhood and global patterns. These learned representations go far beyond traditional pairwise similarities, enabling GNNs to uncover latent collaborative signals through multi-hop message passing. This capability has made GNNs particularly attractive for recommendation scenarios involving intricate user–item and item–item dependencies, where traditional models often fail to fully capture the underlying relational semantics~\citep{ying2018graph,wang2019neural}.

However, despite the recent success, most existing industrial applications of GNNs in recommender systems still adopt a two-stage training paradigm: GNNs are first trained offline to generate static item or user embeddings, which are then fed into a downstream recommendation model.
Although this decoupled design simplifies implementation and enables each component to be optimized independently, it introduces two fundamental drawbacks: First, \textit{high computational overhead.}
Since GNN embeddings are generated offline, the inference must be repeatedly executed to refresh embeddings for all nodes as the underlying graph evolves. In dynamic environments such as TikTok’s recommendation system, where new content and user interactions emerge continuously, this process necessitates frequent recomputation of billions of node embeddings, resulting in substantial infrastructure costs and increased latency. Second and more critically, \textit{lack of joint optimization.}
In the two-stage paradigm, and the recommendation system cannot provide direct gradient feedback to refine the GNN’s learned representations. This disconnect limits the expressive power of the overall system and often results in suboptimal performance.

To address these challenges, we propose \textbf{E2E-GRec}, a novel end-to-end joint training framework that unifies GNN representation learning and recommendation modeling. Unlike traditional approaches, E2E-GRec first samples a subgraph from a heterogeneous cross-domain graph using importance sampling. Then, to mitigate task misalignment, we introduce a Graph Feature Auto-Encoder (GFAE) as a complementary self-supervised learning (SSL) objective that reconstructs node features, thereby providing a direct and stable learning signal to guide the learning of GNN.
More importantly, to enable effective gradient flow from the recommendation model to the GNN in end-to-end training, we first employ two-level feature fusion that strengthens the interaction between graph representations and recommendation features.
We then adopt Gradnorm to adaptively balance the graph learning and recommendation gradients over shared parameters, preventing task dominance and ensuring stable convergence.
Our framework is built upon four key contributions:

\begin{itemize}
[leftmargin=*,noitemsep,topsep=0pt]
\itemsep0em

\item \textbf{Efficient subgraph construction.}Task-specific subgraphs are efficiently extracted from large-scale cross-domain graphs via importance sampling, enabling scalable and adaptive graph construction tailored to recommendation tasks.

\item \textbf{GNN Multi-Task learning.} A Graph Feature Auto-Encoder (GFAE) is introduced as an auxiliary SSL objective, guiding the GNN to learn structurally meaningful embeddings that are subsequently jointly optimized with the recommendation objectives.

\item \textbf{Two levels feature fusion and dynamic loss balancing.} GNN representations are integrated with recommendation features at both bottom and upper levels through complementary gating and attention based fusion strategies. Gradnorm is further employed to dynamically balance the SSL and recommendation losses, mitigating dominance and ensuring stable convergence.

\item \textbf{Comprehensive validation.} Extensive theoretical analysis, offline experiments, online A/B testing, and ablation studies collectively validate the effectiveness of the framework. Significant improvements observed on real-world production dataflow highlight its strong practical value.

\end{itemize}

\section{Related Work}
\label{sec:related}

\subsection{Recommender Systems}

Recommender systems are the cornerstone of modern industrial platforms such as e-commerce, online advertising, and content delivery \citep{zhang2019deep, covington2016deep}.
A typical recommendation pipeline consists of three stages: candidate generation, ranking, and re-ranking.
The candidate generation stage retrieves a small subset of potentially relevant items from billions of candidates using collaborative filtering, approximate nearest neighbor search, or large-scale retrieval models \citep{grbovic2018real, yi2019sampling}.
The ranking stage then scores these candidates with more sophisticated models to optimize engagement metrics such as click-through rate (CTR) or dwell time.
Finally, the re-ranking stage refines the results by incorporating business objectives 
such as diversity, fairness, and monetization \citep{xia2017adapting}.

Learning-to-Rank (LTR) approaches have become fundamental in the ranking stage of industrial recommendation systems, where the primary goal is to optimize the ordering of candidate items presented to users. Within this component, traditional pointwise approaches such as logistic regression and gradient boosted decision trees (GBDT) have been widely adopted in industry, with frameworks like XGBoost and LightGBM showing strong performance in production systems at LinkedIn and Microsoft \citep{chen2016xgboost, ke2017lightgbm}.
Pairwise methods including RankNet \citep{burges2005learning} and LambdaMART \citep{burges2010ranknet} learn relative item preferences, while listwise methods such as ListNet \citep{cao2007learning} and SoftRank \citep{taylor2008softrank} directly optimize ranking metrics (e.g., NDCG), achieving superior results in large-scale commercial systems.

\subsection{Graph Neural Networks in Industrial Recommender Systems}

Graph Neural Networks (GNNs) have emerged as powerful tools to capture higher-order connectivity and context in recommendation scenarios, where user–item interactions form complex graphs. Early industrial adoptions include PinSage \citep{ying2018graph} deployed at Pinterest, which scales GraphSAGE with random-walk sampling and neighborhood importance weighting to generate web-scale item embeddings. Similarly, LightGCN \citep{he2020lightgcn} simplify GCNs by focusing on user–item bipartite structures, achieving state-of-the-art performance in large-scale recommendation benchmarks.

In practice, large companies have developed proprietary GNN-based recommendation frameworks. AliGraph \citep{zhu2019aligraph} at Alibaba and ByteGNN \citep{zheng2022bytegnn} at ByteDance introduce distributed training and heterogeneous graph modeling to support billion-scale data, ensuring both offline training efficiency and online serving latency. Other works such as Euler (Alibaba) and DistDGL (Amazon) \cite{zheng2020distdgl} focus on distributed GNN training frameworks to meet industrial throughput requirements. However, most industrial GNN-based recommendation models  are not trained in a fully end-to-end manner. In such pipelines, the ranking gradients cannot directly influence the GNN parameters, leading to a mismatch between the objectives used for graph representation learning and the final goals optimized in ranking or re-ranking. This objective misalignment often results in suboptimal performance.

Although a few studies have attempted to enable end-to-end training by cascading GNNs with ranking models \cite{li2019graph}, these approaches typically treat the GNN outputs as augmented features that are fused with the ranking network. Consequently, the downstream task remains mismatched—while the GNN component focuses on capturing structural semantics, the ranking model instead aims to optimize task-specific metrics such as click-through rate (CTR) or relevance ranking scores.

\section{Preliminary}
\label{sec:prelim}

\textit{Notation.}
Let $G=(V,E)$ be a graph with $|V|=n$ nodes and
$|E|=m$ edges. We denote by $X\in\mathbb{R}^{n\times d}$ the node-feature
matrix, where the $i$-th row $x_i^\top$ is the $d$-dimensional feature of node
$i$. The adjacency matrix is $A\in\mathbb{R}^{n\times n}$,
with $A_{ij}>0$ iff $(i,j)\in E$; $D=\mathrm{diag}(A\mathbf{1})$ is the degree
matrix. We use $\tilde{A}=A+I$ and $\tilde{D}=D+I$ for self-loop augmentation,
and write the symmetric normalized adjacency matrix
$\hat{A}=\tilde{D}^{-1/2}\tilde{A}\tilde{D}^{-1/2}$ and the normalized Laplacian $L=D-A$ / $\hat{L}=I-\hat{A}$. For a node $u$, its
neighborhood is $\mathcal{N}(u)=\{v:(u,v)\in E\}$, and the $L$-hop neighborhood
is $\mathcal{N}^{(L)}(u)$ defined recursively. When neighbor sampling is used,
$\mathcal{S}_\ell(u)\subseteq\mathcal{N}(u)$ denotes the sampled neighbor multiset
at layer $\ell$.

\noindent A $L$-layer GNN encoder is $f_\theta:\,(X,A)\mapsto H^{(L)}\in\mathbb{R}^{n\times d_L}$,
with hidden representations $H^{(\ell)}=[h^{(\ell)}_1;\dots;h^{(\ell)}_n]$ and
$H^{(0)}=X$. A generic message-passing layer is
\begin{align}
h^{(\ell)}_u = \psi^{(\ell)}\!\Big(
    h^{(\ell-1)}_u,\, \textit{Agg}_{v\in\mathcal{N}(u)}
    (h^{(\ell-1)}_v,\,A_{uv})
\Big),
\end{align}
where $\textit{Agg}$ is an aggregation operator (e.g., mean/sum/max/attention), and $\psi^{(\ell)}$ the update function.

\subsection{GNN for Recommendations}
\label{sec:model}

In this section, we introduce two classical GNN architectures that are employed in our implementation owing to their well-established advantages in both performance and efficiency. These models have been extensively validated in prior research and are widely adopted in large-scale recommendation systems.

\noindent \textbf{GraphSAGE\cite{hamilton2017inductive}.}
GraphSAGE performs inductive neighborhood aggregation with learnable
transformations and (optionally) sampled neighbors. The initial node
representation is set as $h^{(0)}=X$ when raw features are available; otherwise
we initialize $h^{(0)}$ via a trainable embedding lookup. At layer $l$, for a
node $u$ we sample a multiset $\mathcal{S}_\ell(u)\subseteq
\mathcal{N}(u)$ of size $m_\ell$ and compute
\begin{align}
\bar{h}^{(l)}_{\mathcal{N}(u)} &= \frac{1}{|\mathcal{S}_\ell(u)|}\!
\sum_{v\in \mathcal{S}_\ell(u)} h^{(l)}_v,\\
h^{(l+1)}_u &= \sigma\!\bigg(
W^{(l)} \cdot \mathrm{CONCAT}\!\big(h^{(l)}_u,\ \bar{h}^{(l)}_{\mathcal{N}(u)}\big)
\bigg),
\end{align}
where $W^{(l)}$ is a trainable matrix and $\sigma(\cdot)$ a nonlinearity. Other
aggregators (sum/max/LSTM/attention) can be substituted for the mean operators.

\noindent \textbf{LightGCN \cite{he2020lightgcn}.}  
LightGCN simplifies the design of graph convolution by discarding feature
transformation and nonlinear activations, and focuses only on neighborhood
aggregation. $h^{(0)}$ is initialized by a trainable embedding lookup
table. At layer $l$, the embeddings of a user $u$ and an item $i$ are
updated by averaging over their neighbors with degree normalization:
\begin{align}
    h^{(l+1)}_u &= \sum_{i \in \mathcal{N}(u)} 
    \frac{1}{\sqrt{|\mathcal{N}(u)|}\sqrt{|\mathcal{N}(i)|}} \, h^{(l)}_i, \\
    h^{(l+1)}_i &= \sum_{u \in \mathcal{N}(i)} 
    \frac{1}{\sqrt{|\mathcal{N}(i)|}\sqrt{|\mathcal{N}(u)|}} \, h^{(l)}_u.
\end{align}
In matrix form, the propagation can be compactly written as
\begin{align}
    H^{(l+1)} &= \hat{A}\, H^{(l)}, \quad l=0,\dots,L-1, \\
    Y &= \frac{1}{L+1}\sum_{l=0}^L H^{(l)},
\end{align}
where $Y$ is the final representation obtained by
averaging embeddings across all layers.

\subsection{End-to-end Training}
\label{sec:e2e}

In most industrial applications of GNNs for recommender systems, the GNN component is typically trained \textit{offline}, and the resulting node embeddings are then used asaugmented features for the downstream ranking model~\cite{ying2018graph, wang2018billion, mao2021ultragcn}. However, in such decoupled architectures, the gradients from the ranking model cannot be back-propagated into the GNN, resulting in completely decoupled optimization between the GNN encoder and the recommendation model. 
This gradient isolation leads to several drawbacks: (i) the GNN is optimized solely for structural reconstruction rather than task-specific ranking discrimination; (ii) the feature space learned by the GNN may not align well with the recommendation model’s objectives; and (iii) updating the GNN embeddings requires costly offline retraining whenever the user–item distribution changes. 

Here, we present a theorem with a formal proof regarding gradients to theoretically demonstrate why end-to-end learning is critical for aligning the GNN and recommendation objectives. The proof can be found in Appendix~\ref{app:proof1}.

\begin{theorem}[Gradient Coupling in E2E-GRec]
\label{thm:e2e-coupling}
Let $h_\theta(\cdot;\mathcal{G})$ denote the GNN embeddings with parameters $\theta$, $z_i=[h_\theta(x_i;\mathcal{G}) \Vert b_i]$, $s_\psi$ be the recommendation scorer and 
\begin{equation}
J(\theta,\psi,\alpha)=\tfrac{1}{n}\sum_{i=1}^n\Big[\alpha_1 L_{\mathrm{rec}}(y_i,s_\psi(z_i))+\alpha_2 L_{\mathrm{gnn}}(\theta)\Big]
\end{equation}
Assume $\frac{\partial s_\psi(z)}{\partial h_\theta} \neq 0$. Then:

(i) $\nabla_\theta J$ contains a nonzero contribution from $L_{\mathrm{rec}}$ (Rec → GNN);

(ii) $\frac{\partial}{\partial \theta}\big(\nabla_\psi J\big) \neq 0$, i.e., the recommendation model gradient depends on $\theta$ (GNN → Rec).

\noindent In contrast, cascaded pipelines satisfy $\nabla_\theta J_{\mathrm{rec}}=0$ and $\frac{\partial}{\partial \theta}(\nabla_\psi J)=0$, preventing capture of higher-order graph-rec interactions.
\end{theorem}

\section{Methodology}

Motivated by this potential drawback, we introduce our \textbf{\method}: \textbf{E}nd-to-\textbf{E}nd Joint Training Framework for \textbf{G}raph Neural Networks and \textbf{Rec}ommender Systems in this section. We begin by describing our subgraph construction process from the cross-domain graph in Section \ref{sec:graph}. Next, we discuss the potential challenges of integrating GNNs with recommendation systems. We then propose our multi-task learning GNNs designed to address these issues in Section \ref{sec:ssl} Finally, we present the core component of our framework—the effective integration of GNNs and recommendation systems through joint optimization and feature fusion in Section \ref{sec:combine}.

\subsection{Subgraph Construction from Cross Domain Graph}
\label{sec:graph}

To comprehensively capture complex user interests and content associations across diverse business scenarios, we first construct a large-scale weighted cross-domain heterogeneous graph (\textbf{CDG}) \cite{zhu2025ttgl} as the foundation for our subgraph sampling process. Formally, the CDG integrates multi-type nodes and edges across multiple business domains, including video, search, live streaming, and e-commerce. The graph incorporates two categories of relational connections. The first is explicit connections, which are directly derived from observable user behaviors—such as likes, shares, and follows in the video (item) scenario. The second is implicit connections, mined through algorithms that uncover semantic relationships between nodes to capture higher-order behavioral co-occurrences. A key advantage of this heterogeneous graph design lies in its flexibility. Unlike traditional graph learning approaches that rely on manually defined meta-paths requiring domain expertise—and are therefore difficult to adapt to rapidly evolving business environments—our CDG can flexibly incorporate scenario-specific relation types or meta-paths that align with different business objectives and evolving recommendation contexts, without being constrained by fixed or predefined connection rules.

In this work, we focus specifically on the implicit item-to-item (i2i) relation construction as the core foundation. This is because, first, i2i graphs have consistently demonstrated effectiveness, low noise, and high quality in our prior experiments on recommendation systems (e.g., retrieval stage). In addition, given the large-scale industrial setting, we construct homogeneous item-to-item relations, which is more beneficial for maintaining overall model, data, and system simplicity while ensuring training and serving efficiency. Nevertheless, as emphasized, different business scenarios and stages can flexibly introduce additional meta-paths or relation types to build graphs according to their specific needs.

Specifically, we built this i2i relations based on the Swing \cite{yang2020large} algorithm, which models substitute relationships by exploring user–item –user–item “swing” structures in the interaction graph. For items $i$ and $j$, Swing traverses all user pairs $(u,v)$ who have clicked both items, with each user pair contributing a score of $\frac{1}{\alpha + |I_u \cap I_v|}$, where $|I_u \cap I_v|$ represents the number of common items clicked by both users. This design cleverly penalizes user pairs with overly overlapping click patterns, as they might just be highly active users browsing randomly. Additionally, the algorithm introduces a user weighting factor $w_u = \frac{1}{\sqrt{|I_u|}}$ to reduce the influence of highly active users. Compared with traditional similarity measures, Swing leverages richer structural information and effectively filters out noisy data and captures more reliable product similarity relationships, constructing high-quality product substitute relationship graphs.

Due to the large-scale nature of the CDG, leveraging the complete graph is computationally prohibitive for online serving. Hence, we sample i2i subgraphs from the CDG, which is refreshed daily from real-time data pipeline to capture the evolving item relationships. To address this, we adopt an importance sampling strategy that builds multi-hop subgraphs for both training and inference. Specifically, given a source item $u$ which is from input sequence and sampling depth $L$, the probability of sampling a neighbor $v \in \mathcal{N}(u)$ at layer $\ell$ is defined as
\begin{equation}
p_{u \to v}^{(\ell)} =
\frac{\big(\omega_{uv}\big)^\beta}
{\sum_{k \in \mathcal{N}(u)} \big(\omega_{uk}\big)^\beta},
\end{equation}
where $\omega_{uv}$ denotes the edge weight (e.g., Swing score) between items $u$ and $v$, $\mathcal{N}(u)$ is the neighbor set of $u$ in the CDG, and $\beta \geq 0$ is a temperature parameter that controls the sampling concentration. The number of sampled neighbors is further constrained by a hyperparameter $k$, ensuring tractable time complexity. Note that the subgraph is constructed in real time, which strikes a balance between graph completeness and computational efficiency, enabling scalable and responsive recommendations.

\subsection{GNN Multi-Task Co-Optimization}
\label{sec:ssl}

Given a subgraph, the most straightforward way to achieve end-to-end training is to place a GNN before the recommendation model, forming a supervised cascaded architecture \cite{li2019graph}. Specifically, the subgraph is first fed into the GNN, and the resulting node embeddings are concatenated with other features as inputs to the downstream recommendation model. While this approach is simple and easy to implement, it suffers from several significant drawbacks: 

\begin{itemize}
[leftmargin=*,noitemsep,topsep=0pt]
\itemsep0em

\item \textbf{Long backpropagation path yields superficial gradient signal:} In this setup, the GNN receives its learning signal solely from the final recommendation objective. This supervision must propagate backward through the entire recommendation module, leading to weak guidance for the GNN module. As the gradient passes through multiple nonlinear transformations, it tends to vanish or become unstable, resulting in a shallow and noisy signal by the time it reaches the GNN parameters. Formally, the gradient reaching GNN parameters is computed as:
\begin{equation}
\frac{\partial \mathcal{L}_{\mathrm{rec}}}{\partial \theta_{\mathrm{GNN}}}
=
\underbrace{\frac{\partial \mathcal{L}_{\mathrm{rec}}}{\partial o}}_{\text{loss-to-output}}
\; \cdot
\underbrace{\frac{\partial o}{\partial z}}_{J_g(z) = \prod_{l=1}^{L} \frac{\partial f_l}{\partial f_{l-1}}}
\; \cdot
\underbrace{\frac{\partial z}{\partial \theta_{\mathrm{GNN}}}}_{\text{GNN-internal Jacobian}}
\end{equation}
where $o = g(z) = f_L \circ f_{L-1} \circ \cdots \circ f_1(z)$ denotes the recommendation model composed of $L$ nonlinear layers.  This product of Jacobians can either vanish (if singular values < 1) or explode (if singular values > 1), making training unstable.

\item \textbf{Potential Task Misalignment between GNN and Recommendation:}
The second critical issue is the inherent task misalignment between GNN and recommendation systems. The primary objective of a recommendation system is to present users with an optimal set or ordering of items that maximizes user satisfaction or engagement. This task is inherently discriminative and relative—it focuses on comparing candidate items to identify and rank those most relevant to a user’s preferences.
In contrast, GNN aims to learn meaningful graph topology representations and node embeddings that capture structural properties such as local neighborhood patterns and graph properties, which is different from recommendation systems. 

\end{itemize}

According to the above analysis, we need to design an effective auxiliary graph learning task to train the GNN, rather than simply cascading it with the recommendation module. However, since explicit labels
for training the GNN on a specific graph learning task are unavailable, we propose a self-supervised learning (SSL) task specifically formulated as a generative objective that reconstructs node features. This auxiliary SSL objective provides a stable and informative learning signal for the GNN, allowing it to capture the intrinsic structural semantics of the graph. We first introduce the technical details of our design, followed by a theoretical analysis that demonstrates the advantages over cascaded supervised learning frameworks.

\noindent \textbf{Graph Feature Auto-Encoder (GFAE)} 
Graph self-supervised learning (SSL) methods can be broadly categorized into contrastive and generative approaches. Contrastive methods aim to bring representations of similar nodes (or augmented graph views) closer together while pushing apart those of dissimilar ones. In contrast, generative approaches train models to reconstruct original node features or predict the existence or weights of edges.

Since our goal is to enable end-to-end learning while maintaining efficiency for online serving, we adopt the generative paradigm via a graph auto-encoder framework. This choice is motivated by the fact that contrastive learning typically requires complex training strategies—such as the careful construction of negative samples and reliance on high-quality data augmentations \cite{xue2025h} — which can hinder scalability and deployment efficiency.

Specifically, we employ a variant of the \textbf{Graph Auto-Encoder (GAE)}. 
Instead of reconstructing the adjacency matrix, we adopt a mean squared error (MSE) loss to reconstruct the input node features.

This design is motivated by the need to obtain more informative node representations that can be effectively concatenated with recommendation features for joint optimization. (See Section~\ref{sec:combine}).
Formally, given an input graph $\mathcal{G} = (\mathcal{V}, \mathcal{E})$ with node features $X \in \mathbb{R}^{n \times d}$, our Graph feature Auto-Encoder consists of an encoder $f_\theta(\cdot)$ and a decoder $g_\phi(\cdot)$ with the reconstruction loss as:
\begin{align}
\mathcal{L}_{\mathrm{ssl}} = \| X - \hat{X} \|_F^2 = \| X - g_\phi(f_\theta(X, A)) \|_F^2.
\end{align}

In traditional settings, $f_\theta$ is often implemented as a GNN (e.g., LightGCN or GraphSAGE), and the decoder $g_\phi$ reconstructs the node features $\hat{X}$ through an MLP. To further reduce computational complexity, we adopt a lightweight decoder design and set $g_\phi(\cdot)$ to be the identity mapping:
\begin{align}
f_\theta(X, A) = GNN(X, A), \quad g_\phi(Z) = Z,
\end{align}

Despite its simplicity, this configuration empirically achieves strong performance in our experiments, indicating that the encoder (GNN) itself learns sufficiently expressive node embeddings to recover the original features without requiring an additional decoder.

Next, we provide a theoretical result to demonstrate why introducing a separate SSL task for the GNN is more effective than simply using the GNN output as an augmented feature and training the downstream recommendation model in a cascaded manner:

\begin{theorem}[SSL vs.\ Cascaded: Objective Misalignment]\label{thm:ssl_cascade_misalignment}
Let $X\in\mathbb{R}^{n\times d}$ be node features and $Z=f_\theta(X)\in\mathbb{R}^{n\times k}$ $(k\!\ge\! d)$
the encoder output. Assume a linear decoder $W_d\in\mathbb{R}^{k\times d}$ with full column rank
$\mathrm{rank}(W_d)=d$. If the SSL reconstruction objective attains zero loss,
\begin{align}
\|X - ZW_d\|_F^2 = 0,
\end{align}
then $\mathrm{col}(X)\subseteq \mathrm{col}(Z)$ (i.e., $Z$ preserves the full feature subspace of $X$).
For simplicity, consider the cascaded recommendation head with BPR scores $s_i = w^\top z_i$ for some $w\in\mathbb{R}^k$ and
\begin{align}
\mathcal{L}_{\mathrm{BPR}}(Z)=\sum_{(i,j)\in\mathcal{P}} \ell\!\big(w^\top(z_i - z_j)\big),
\end{align}
where $\ell'(\cdot)$ exists. If $\mathrm{col}(X)\nsubseteq \mathrm{span}(w)$, then the two objectives are misaligned:
there exists a nontrivial subspace
\begin{align}
\mathcal{U}\;=\;\mathrm{col}(X)\cap \ker(w^\top)\;\neq\;\{0\}
\end{align}
such that SSL requires $Z$ to preserve information along $\mathcal{U}$, while BPR is invariant to any perturbations $\{\delta z_i\}$ with
$\delta z_i\in \ker(w^\top)=\{v\in\mathbb{R}^k: w^\top v=0\}$ for all $i$ (hence supplies no gradient on $\mathcal{U}$).
Therefore BPR constrains only the one-dimensional subspace $\mathrm{span}(w)$ and leaves the $(k-1)$-dimensional orthogonal subspace
$\ker(w^\top)$ unconstrained, establishing the misalignment.
\end{theorem}

Our theorem theoretically demonstrates that the proposed joint training framework yields more informative embeddings. The proof can be found in Appendix~\ref{app:proof2}.

\begin{figure*}[t]
\centering
\includegraphics[width=0.73\textwidth]{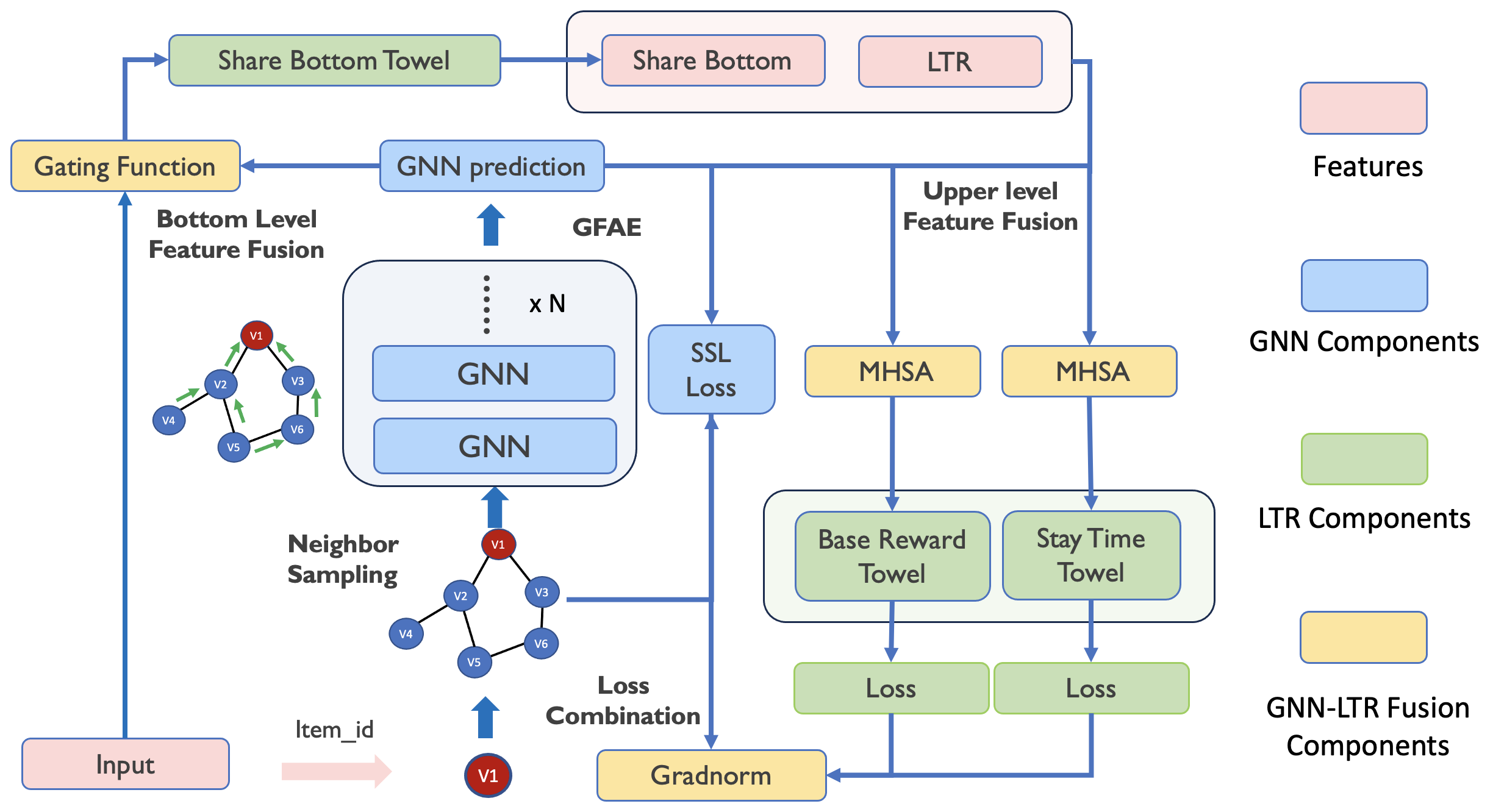}
\caption{Overview of \method. Colors indicate distinct functional blocks. 
(1) \textbf{Subgraph sampling:} Given a source item ID, we sample $k$-hop neighbors to form a subgraph (Sec. \ref{sec:graph}); 
(2) \textbf{GNN SSL:} GNN trained via a graph autoencoder (Sec. \ref{sec:ssl});
(3) \textbf{GNN–LTR fusion:} GNN representations are integrated and jointly optimized with LTR in an end-to-end manner. (Sec. \ref{sec:combine}) }
\label{fig:framework}
\vspace{-3mm}
\end{figure*}

\subsection{End-to-End Joint Training for GNN and Recommendation Systems}
\label{sec:combine}

As introduced above, the key challenge lies in effectively integrating GNN components with recommendation systems. To achieve this, we combine the two modules from two complementary perspectives: \textit{loss combination} and \textit{feature fusion}. Note that our proposed framework can be applied to various stages of recommendation systems. In this work, we specifically focus on the Learning-to-Rank (LTR) component, which serves as one of the fundamental modules in the ranking stage of industrial recommendation pipelines. An overview of the proposed framework is presented in Figure \ref{fig:framework}.

\subsubsection{\textbf{Multi-Task Loss Combination}}

To enable end-to-end optimization between the GNN encoder and the LTR model, we jointly optimize two complementary objectives:
the self-supervised reconstruction loss $\mathcal{L}_{\mathrm{SSL}}$ for GNN representation learning and
the supervised loss $\mathcal{L}_{\mathrm{LTR}}$ for downstream prediction.

\noindent \textbf{Overall Loss.}
The total objective is a weighted combination of the two tasks:
\begin{align}
\mathcal{L}_{\mathrm{total}} = w_{\mathrm{SSL}}(t)\, \mathcal{L}_{\mathrm{SSL}} + w_{\mathrm{LTR}}(t)\, \mathcal{L}_{\mathrm{LTR}},
\end{align}
where $w_{\mathrm{SSL}}(t)$ and $w_{\mathrm{LTR}}(t)$ are dynamic task weights at training step $t$.
Because of the dynamic update process, the scale of the losses changes over time, Naively fixing $w_{\mathrm{SSL}}$ and $w_{\mathrm{LTR}}$ across all time steps may let one loss dominate the other. To avoid this, we apply
Gradnorm~\cite{chen2018gradnorm} for adaptive gradient balancing based on shared parameter space.

\noindent \textbf{Gradient-Based Task Weighting.}
For each task $i \in \{\mathrm{SSL}, \mathrm{LTR}\}$, we define its gradient norm with respect to the shared parameters $\theta_s$ (the input source node feature $h^{(0)}$ in our case):
\begin{align}
G_i(t) = \left\|\nabla_{\theta_s}\big(w_i(t)\, \mathcal{L}_i(t)\big)\right\|_2.
\end{align}
The goal of GradNorm is to ensure all tasks train at a similar rate by adjusting $w_i(t)$ to maintain proportional gradient magnitudes.
The relative training rate of task $i$ is computed as:
\begin{align}
r_i(t) = \frac{\mathcal{L}_i(t) / \mathcal{L}_i(0)}{\frac{1}{T}\sum_{j=1}^T \mathcal{L}_j(t) / \mathcal{L}_j(0)},
\end{align}
where $\mathcal{L}_i(0)$ is the initial loss and $T$ is the total number of tasks.
Then, Gradnorm minimizes the following objective to align the gradient magnitudes:
\begin{align}
\mathcal{L}_{\mathrm{Gradnorm}} =
\sum_{i \in \{\mathrm{SSL},\,\mathrm{LTR}\}}
\left|\, G_i(t) - \bar{G}(t) \cdot r_i(t)^\gamma \, \right|,
\end{align}
where $\bar{G}(t)$ is the mean gradient norm across tasks, and $\gamma$ controls how strongly faster tasks are slowed down. In our experiment, $\gamma$ is typically set to $1$, which has been found empirically effective
for maintaining stable multi-task learning and balanced convergence rates.
This mechanism adaptively increases $w_i(t)$ for slower-learning tasks and decreases it for dominating ones, keeping both $\mathcal{L}_{\mathrm{SSL}}$ and $\mathcal{L}_{\mathrm{LTR}}$ at balanced learning rates.

\noindent \textbf{A Dual-Objective Optimization.}
The complete training is a dual-objective optimization process where model parameters and task weights are updated in parallel based on different objectives. At each training step, we perform two distinct updates: 
\textit{(1) Model Parameter Update:} The parameters of the neural network, $\theta = \{\theta_s, \theta_{\mathrm{GNN}}, \theta_{\mathrm{LTR}}\}$, are updated by minimizing the main task loss, $\min_{\theta} \mathcal{L}_{\mathrm{total}}(t)$. \textit{(2) Task Weight Update:} The dynamic task weights, 
$w_i = \{w_{\mathrm{SSL}}, w_{\mathrm{LTR}}\}$, are updated by minimizing 
the Gradnorm meta-loss $\min_{w_i > 0} \; \mathcal{L}_{\mathrm{GradNorm}}(t).$
Here $w_i$ are normalized after each update to keep $\sum_i w_i = T$, following \cite{chen2018gradnorm}.

This joint optimization framework ensures that the GNN continuously refines its representations under supervision from both
structural (SSL) and LTR signals. By separating the optimization goals, the model learns task-aligned embeddings in a more stable and balanced manner.

\subsubsection{\textbf{Two-level Feature Fusion}}

Besides the loss function, the GNN-derived embeddings play a crucial role, as they capture high-order collaborative filtering signals over the graph structure. To effectively inject this graph information into the recommendation model and enhance its capacity, we fuse the GNN embeddings with other feature sources (e.g., user profile features, LTR). As illustrated in Figure \ref{fig:framework}, we integrate the GNN-derived features at two hierarchical levels—the \textit{bottom} and the \textit{upper} levels.
At the bottom level, we combine the GNN output with other input features and feed the fused representation into a shared bottom tower. The output of this tower serves as the fundamental feature foundation for the two downstream Learning-to-Rank (LTR) branches: the base-reward tower and the stay-time tower. At the upper level, we re-inject the GNN features together with the task-specific LTR feature along with the output from the shared bottom tower. This design provides an extra gradient flow that travels through a shorter path with less attenuation and enhances the model’s high-order perceptive capacity, enabling more expressive interactions between graph-derived and task-specific features across different semantic levels.

However, simple concatenation is insufficient to capture the complex interrelations between these features. To address this limitation, we explore two complementary fusion strategies: 
(1) \textbf{Concatenation with gating}, which provides simple yet effective feature-wise control, 
and (2) \textbf{Attention-based token fusion}, which treats each feature source as a token and performs multi-head attention across them.
Let the GNN output be $F_{\mathrm{gnn}} \in \mathbb{R}^{N \times d}$ and the remaining feature groups 
$\{F_{\mathrm{shared}}, F_{\mathrm{ltr}}\}$ each in $\mathbb{R}^{N \times d}$, where $
F_{\mathrm{shared}}
= \mathrm{Tower_{share}}\allowbreak \!\big(\,[F_{\mathrm{in}},\, F_{\mathrm{gnn}}]\,\big)$ is the output of the share bottom towel. 
\paragraph{1. Concatenation with Gating.}
We first concatenate all feature groups together
$F_{\mathrm{cat}}   = \textsc{Concat}\big(\,F_{\mathrm{shared}},\, F_{\mathrm{ltr}},\, F_{\mathrm{gnn}}\,\big)$. To adaptively control the contribution of different feature sources, we introduce a gating function that learns task-relevant feature weights:
\begin{align}
\mathbf{g} = \sigma\left(\mathbf{W}_g F_{\mathrm{cat}} + \mathbf{b}_g\right), \quad
F_{\mathrm{fused}}^{(\mathrm{gate})} = \mathbf{g} \odot F_{\mathrm{cat}}
\end{align}

\noindent where $\mathbf{W}_g \in \mathbb{R}^{d \times d}$ and $\mathbf{b}_g \in \mathbb{R}^d$ are learnable parameters, $\sigma(\cdot)$ is the sigmoid activation, and $\odot$ denotes element-wise multiplication. This gating mechanism allows the model to dynamically emphasize or suppress specific features based on their relevance to the task.

\paragraph{2. Attention-based Token Fusion.}
To capture inter-feature dependencies and further refine the fused representation, we apply Multi-Head Self-Attention(\textbf{MHSA}) over the stacked feature sequence $F_{\mathrm{stack}} =\textsc{Stack}\big(\,F_{\mathrm{shared}},\, F_{\mathrm{ltr}},\, F_{\mathrm{gnn}}\,\big)$.
Specifically, we treat each feature as an individual token:
\begin{equation}
F_{\mathrm{fused}}^{(\mathrm{attn})} = \mathrm{Mean}_{[:,t,:]}\big( \textbf{MHSA}(F_{\mathrm{stack}}) + F_{\mathrm{stack}}\big), \quad F_{\mathrm{stack}} \in \mathbb{R}^{N \times T \times d},
\end{equation}

\noindent where the residual connection ensures stable gradient flow. The refined features $F_{\mathrm{fused}}^{(\mathrm{attn})}$ (or $F_{\mathrm{fused}}^{(\mathrm{gate})}$) are then fed into downstream base-reward/stay-time towers:
\begin{equation}
F_{\mathrm{reward}} = \mathrm{Tower}_{\mathrm{reward}}\!\big(F_{\mathrm{fused}}\big), 
\quad
F_{\mathrm{stay}}   = \mathrm{Tower}_{\mathrm{stay}}\!\big(F_{\mathrm{fused}}\big).
\end{equation}
\begin{equation}
F_{\mathrm{final}} = \beta_{\mathrm{r}} \cdot F_{\mathrm{reward}} + \beta_{\mathrm{s}} \cdot F_{\mathrm{stay}},
\end{equation}
Here, $F_{\mathrm{reward}}$ and $F_{\mathrm{stay}}$ denote the outputs of the base-reward and stay-time towers, respectively,
and their weighted sum forms the final LTR prediction $F_{\mathrm{final}}$, where $\beta_{\mathrm{r}}$ and $\beta_{\mathrm{s}}$ is a dynamic weighting coefficient that balances the contribution between the base-reward prediction and the stay-time signal. The GNN parameters are optimized through a multi-task loss (also see Section \ref{sec:ssl}):
\begin{align}
\nabla_{\theta}\mathcal{L}_{\mathrm{total}}
  &= w_{\mathrm{SSL}}\nabla_{\theta}\mathcal{L}_{\mathrm{SSL}}
   + w_{\mathrm{r}}J_{\theta\!\to Z}^{\top}\,
      \frac{\partial \mathcal{L}_{\mathrm{reward}}}{\partial F_{\mathrm{gnn}}}
   + w_{\mathrm{s}}J_{\theta\!\to Z}^{\top}\,
      \frac{\partial \mathcal{L}_{\mathrm{stay}}}{\partial F_{\mathrm{gnn}}},
\end{align}
In our experiments, the attention-based token fusion achieved the best overall performance, meanwhile, the gating-based fusion serves as an efficient alternative with latency constraints. Considering that the upper-level fusion directly precedes the task-specific towers and thus receives more immediate gradient signals, we adopt attention-based fusion at the upper level to maximize feature interaction and gating-based fusion at the bottom level provides efficient feature modulation while reducing computational overhead. Nevertheless, both fusion modules are fully interchangeable, allowing flexible adaptation to different deployment scenarios or efficiency requirements (such as online serving).

\section{Experiments}
\label{sec:exp}

In this section, we present comprehensive experiments, including offline AUC evaluations across days in Section \ref{sec:offline}, ablation studies on key components in Section \ref{sec:ablation}, and online A/B testing results conducted within practical recommendation systems in Section \ref{sec:online}, to demonstrate the effectiveness of our proposed \method.

\subsection{Offline Results}
\label{sec:offline}

\noindent \textbf{Settings} 
As mentioned earlier, although our framework can be integrated into various stages of a recommendation system,
we focus on the LTR stage in this work. 
The experimental configurations are summarized as follows:

\begin{itemize}[leftmargin=1.5em]
    \item \textbf{Subgraph Sampling.}
    We consider two sampling strategies: 
    (a) sampling one-hop neighbors, where 100 neighbors are sampled for each source node to construct the subgraph; and 
    (b) sampling two-hop neighbors, where 25 and 15 neighbors are sampled for the first and second hops, respectively. 
    During experiments, we found that the one-hop subgraph already provides significant offline improvements while maintaining high efficiency. 
    Therefore, we adopt the one-hop configuration as the default setting in all experiments.

    \item \textbf{GNN Architecture.}
    For the GNN component, we primarily adopt \textsc{LightGCN}~\cite{he2020lightgcn} (as introduced in Section~\ref{sec:model}), which has demonstrated strong effectiveness in large-scale recommendation systems. We also conduct an ablation study on different GNN backbones, as presented in Section~\ref{sec:ablation}.
    The number of GNN layers $k$ is set to match the number of sampled hops 
    (i.e., $k=1$ for the one-hop subgraph described above). 
\end{itemize}
The hyperparameter searching space is provided in Appendix \ref{app:hyper}. 

\noindent \textbf{Evaluation} 
We first convert the continuous \emph{staytime} metric into a binary label $y^{(0)}_i$ based on a predefined threshold~$\tau$:
\begin{equation}
y^{(0)}_i =
\begin{cases}
1, & \text{if } \text{staytime}_i > \tau,\\[4pt]
0, & \text{otherwise.}
\end{cases}
\end{equation}

To enhance the supervision signal, we perform a label refinement step that aggregates multiple user interaction indicators (each typically binary, $0$ or $1$) into a unified interaction score:
\begin{equation}
\mathrm{interact}_i = \sum_{k} \text{pos\_action}_{i,k} - \text{neg\_action}_i.
\end{equation}
If a user exhibits any positive interaction (e.g., like, comment, share), the interaction score increases, while negative feedback decreases it.  
The final label is then obtained by combining the staytime-based label with this interaction score and clipping the result into $[0,1]$:
\begin{equation}
y_i = \mathrm{clip}\big(y^{(0)}_i + \mathrm{interact}_i,\; 0,\; 1\big).
\end{equation}

\noindent
We apply the binary cross-entropy (BCE) loss for the two major towers (base-reward and stay-time) 
and AUC as the primary evaluation metric. 
Instead of absolute AUC values, we report the \textit{relative improvements} 
over the baseline LTR model, where the improvement is computed as the percentage gain 
with respect to the baseline performance. 
Since our offline training is conducted in a \textbf{streaming} manner, to better demonstrate the stability and generalization of our model, 
we report the AUC improvement across six consecutive days after the model has converged, ensuring the fair comparison. We evaluate two fusion strategies for the upper-level feature integration: 
(1) simple gating fusion (\texttt{gate}) and 
(2) attention-based fusion (\texttt{attn}) in the result Table \ref{tab:auc} (see Section \ref{sec:combine}). 

\vspace{-3mm}
\begin{table}[h]
\centering
\caption{Relative AUC improvements (\%) across different models and dates.}
\vspace{-3mm}
\resizebox{0.5\textwidth}{!}{%
\begin{tabular}{lcccccccc}
\toprule
\textbf{Model} & \textbf{Average} & \textbf{Day1} & \textbf{Day2} & \textbf{Day3} & \textbf{Day4} & \textbf{Day5} & \textbf{Day6} \\
\midrule
\method(\texttt{gate})  & +1.40\% & +1.37\% & +1.33\% & +1.40\% & +1.43\% & +1.45\% & +1.44\% \\
\method(\texttt{attn}) & +1.65\% & +1.71\% & +1.56\% & +1.59\% & +1.59\% & +1.60\% & +1.58\% \\
\bottomrule
\end{tabular}
}
\label{tab:auc}
\vspace{-3mm}
\end{table}

From the results, we can observe that our proposed \method{} consistently outperforms the
baseline LTR model. This improvement can be attributed to the following factors:

\begin{enumerate}[leftmargin=1.5em]
    \item \textbf{Higher-order collaborative signals.}
    By introducing graph-based message passing, our model captures higher-order
    collaborative relationships among items, allowing the ranking tower to
    exploit neighborhood information beyond individual interactions. 

    \item \textbf{End-to-end optimization.}
    The end-to-end training framework creates a powerful synergy between the GNN and the downstream recommendation task. This structure allows the gradients from the LTR loss to flow all the way back through the network and directly influence the GNN's parameter updates. This ensures that the GNN also learns task-specific graph representations optimized for LTR task, leading to more discriminative and task-relevant embeddings.

    \item \textbf{Self-supervised learning.}
    The SSL task guides the GNN to learn and generate more informative and higher-quality node representations by focusing on the intrinsic structural properties of the graph.  This approach mitigates potential task misalignment (see Theorem \ref{thm:ssl_cascade_misalignment}) and ultimately provides a more powerful set of embeddings for the downstream recommendation system. 
\end{enumerate}

\subsection{Ablation Study}
\label{sec:ablation}

\subsubsection{GNN backbones}
We first provide an ablation study on different GNN backbones in Figure \ref{fig:backbone}. From the results, we can see that LightGCN significantly outperforms GraphSAGE.
This is because LightGCN removes unnecessary nonlinear transformations and feature mixing, allowing it to focus purely on high-order collaborative signals through linear neighborhood aggregation.
Such a design not only reduces potential over-smoothing issue but also preserves the original embedding space structure, leading to more stable and better representations in recommendation tasks.
\begin{figure}[h]
\centering
\begin{minipage}[t]{0.22\textwidth}
\centering
\includegraphics[width=\linewidth]{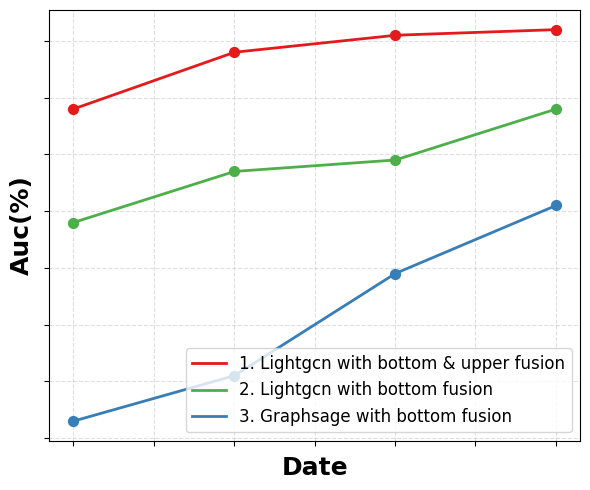}
\caption{Ablation on GNN backbones and fusion.}
\label{fig:backbone}
\end{minipage}
\hfill
\begin{minipage}[t]{0.22\textwidth}
\centering
\includegraphics[width=\linewidth]{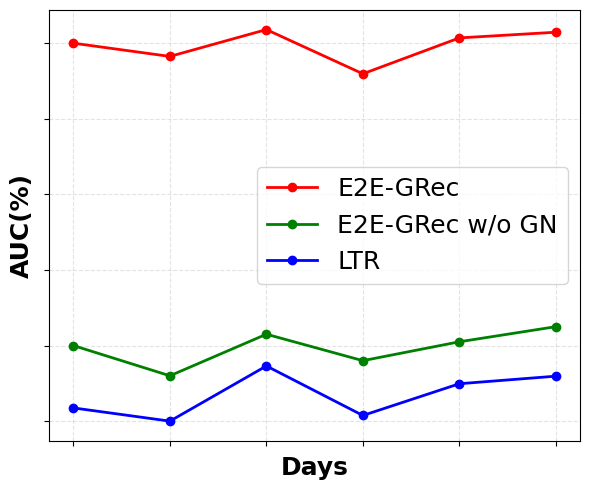}
\caption{Ablation on the effect of Gradnorm.}
\label{fig:gn}
\end{minipage}
\vspace{-3mm}
\end{figure}

\subsubsection{Gradnorm \& SSL}
In this section, we provide an ablation study on the influence of Gradnorm. We simply replace Gradnorm with a fixed coefficient in front of our SSL loss when performing the loss combination, while keeping all other settings the same. The results are shown in Figure \ref{fig:gn}.
From the results, we can observe that:
(1) Without Gradnorm, the GNN module can still enhance the basic recommendation system by incorporating neighbor information.
(2) With Gradnorm, the performance is significantly improved. This is because it adjusts the gradient magnitudes of the SSL and LTR losses according to their learning speeds, ensuring that neither task dominates the optimization. We also provide Figure \ref{fig:weight} to further illustrate this phenomenon, showing how Gradnorm dynamically balances the loss weights during training. As training progresses, Gradnorm adaptively adjusts the relative importance of each objective, ensuring stable optimization and preventing any single task from dominating the learning process. 
\begin{figure}[h]
\centering
\hspace{-5mm}
\includegraphics[width=0.50\textwidth]{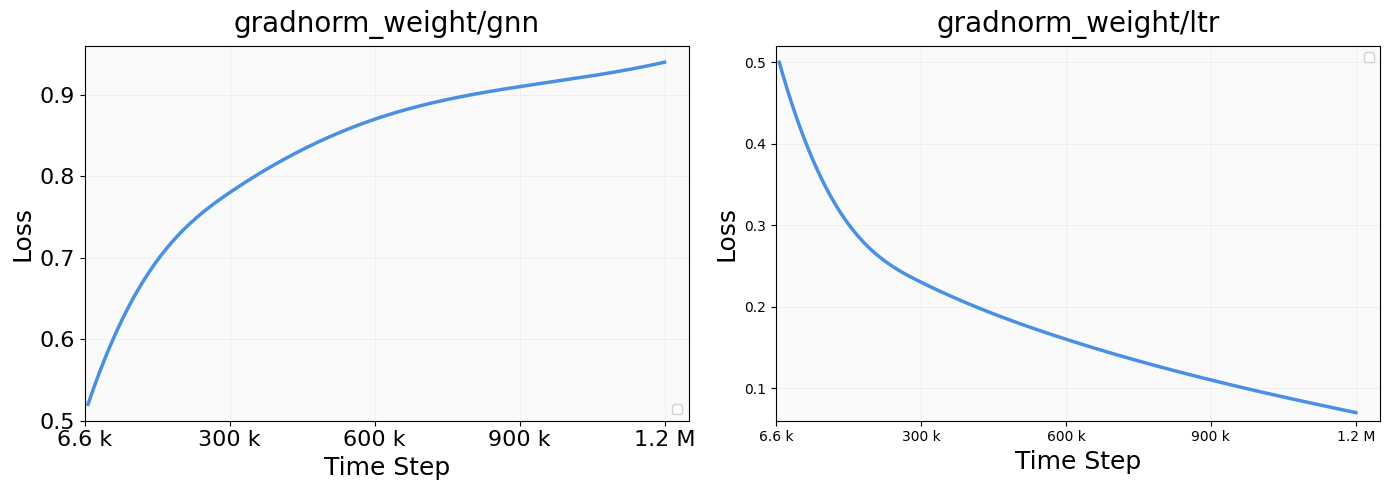}
\caption{Dynamic Weight assigned by Gradnorm}
\label{fig:weight}
\vspace{-5mm}
\end{figure}

\subsubsection{Fusion Strategy} 
As mentioned in Section \ref{sec:combine}, gating and attn serve as two different fusion strategies that trade off between performance and efficiency.
To validate this, we provide an ablation study in Table \ref{tab:auc}.
From the results, we can observe that attention-based fusion consistently outperforms gating-based fusion, which confirms our intuition that the attention mechanism can more effectively capture inter-relations between different feature representations from the LTR and GNN modules.

Furthermore, from Figure \ref{fig:backbone}, we can observe that applying both upper and bottom fusion yields a much larger improvement compared to using bottom fusion alone. Compared with bottom-only fusion, the additional upper fusion path introduces an extra gradient contribution that is shorter and less attenuated, thereby generally providing more effective supervision and leading more stable convergence.

\subsection{Online Results}
\label{sec:online}
To validate our method's real-world effectiveness, we deployed it in the online A/B test on the production large-scale recommendation system. We compared our model against the platform's highly optimized production baseline, which incorporates both rule-based strategies and state-of-the-art models.
The evaluation focused on key user engagement and retention metrics: \textit{StayDuration (SD)}: The average time a user spends on the platform per session, a primary indicator of overall engagement. \textit{Last-7-Active Days (LT-7)}: The number of days a user was active in the last seven days, measuring short-term user retention and habit formation. Skip-related Metrics: \textit{Skip/Play}: The ratio of skipped videos to total videos played. A lower value signifies higher content relevance. \textit{Skip/User}: The average number of videos a user skips, indicating overall user satisfaction.
\textit{PlayTimeRate/Play (PTR)}: The ratio of a video's actual playtime to its total duration, measuring how engaging a specific piece of content is.

Our model demonstrated statistically significant improvements across all key metrics, as shown in Table \ref{tab:ab_lift_summary}. Notably, we observed a \textbf{+0.133\%} relative increase in \textit{StayDuration (SD)} and \textbf{+0.0262\%} relative increase in \textit{Last-7-Active Days (LT-7)}, showing users spent more time on the platform. At the same time, a \textbf{-0.1735\%} reduction in \textit{Skip/Play} and a \textbf{-0.3171\%} reduction in \textit{Skip/User} indicate that our model recommended more relevant content that users were less likely to skip.
These results confirm that our method \method delivers a positive impact on both immediate user engagement and longer-term retention. We attribute this success to the model's ability to leverage higher-order collaborative signals from the graph structure within an end-to-end training framework, leading to more effective and holistic optimization of the recommendation pipeline.

\vspace{-3mm}
\begin{table}[h]
\centering
\caption{Relative Improvement (\%) of Key Metrics in Online A/B Testing}
\vspace{-3mm}
\label{tab:ab_lift_summary}
\resizebox{0.45\textwidth}{!}{%
\begin{tabular}{cccccc}
\toprule
\textbf{Metric} & \textbf{SD} & \textbf{LT-7} & \textbf{Skip/Play} & \textbf{Skip/User} & \textbf{PTR} \\
\midrule
\textbf{Lift (\%)} & +0.133 & +0.0262 & -0.1735 & -0.3171 & +0.1247 \\
\bottomrule
\end{tabular}
}
\vspace{-5mm}
\end{table}

\section{Conclusion}
\label{sec:conclu}

In this paper, we presented E2E-GRec, a novel end-to-end joint training framework that unifies GNNs with industrial recommender systems. Built upon subgraphs sampled from a large-scale, cross-domain graph, our framework first introduces a Graph Feature Auto-Encoder as an auxiliary self-supervised task to enhance the quality of the learned GNN embeddings. The subsequent two-level feature fusion enables the effective injection of graph information into the recommendation model. Furthermore, the Gradnorm-based dynamic loss balancing ensures stable convergence and prevents task dominance during end-to-end training. Extensive offline evaluations, online A/B tests, and ablation studies on practical production data, together with theoretical analysis, demonstrate the superior effectiveness of our proposed approach.

\bibliographystyle{ACM-Reference-Format}
\bibliography{ref_e2e}

\appendix

\newpage
\appendix
\renewcommand{\thetheorem}{\arabic{theorem}} 
\setcounter{theorem}{0} 

\section{Proof of Theorem 1}
\label{app:proof1}

\begin{theorem}[Gradient Coupling in E2E-GRec]
\label{thm:e2e-coupling}
Let $h_\theta(\cdot;\mathcal{G})$ denote the GNN embeddings with parameters $\theta$, $z_i=[h_\theta(x_i;\mathcal{G}) \Vert b_i]$, $s_\psi$ be the recommendation scorer and 
\begin{equation}
J(\theta,\psi,\alpha)=\tfrac{1}{n}\sum_{i=1}^n\Big[\alpha_1 L_{\mathrm{rec}}(y_i,s_\psi(z_i))+\alpha_2 L_{\mathrm{gnn}}(\theta)\Big]
\end{equation}
Assume $\frac{\partial s_\psi(z)}{\partial h_\theta} \neq 0$. Then:

(i) $\nabla_\theta J$ contains a nonzero contribution from $L_{\mathrm{rec}}$ (Rec → GNN);

(ii) $\frac{\partial}{\partial \theta}\big(\nabla_\psi J\big) \neq 0$, i.e., the recommendation model gradient depends on $\theta$ (GNN → Rec).

\noindent In contrast, cascaded pipelines satisfy $\nabla_\theta J_{\mathrm{rec}}=0$ and $\frac{\partial}{\partial \theta}(\nabla_\psi J)=0$, preventing capture of higher-order graph-rec interactions.
\end{theorem}

We show that, unlike cascaded pipelines that pretrain a GNN and then freeze its
embeddings as augmented features for an recommendation model, our end-to-end (E2E) co-training causes the gradients of the GNN and the recommendation model to \emph{mutually influence} each other.
This coupling occurs (i) through feature fusion—where the recommendation model consumes the
GNN embeddings via concatenation—and (ii) through loss fusion—where multiple
objectives are combined via GradNorm.

\paragraph{Setting.}
Let $\mathcal{D}=\{(x_i, y_i)\}_{i=1}^n$ denote training instances
(e.g., user--item pairs or query--document pairs) on a graph $\mathcal{G}$.
A GNN with parameters $\theta$ produces node/item/user embeddings
$h_\theta(\cdot;\mathcal{G})$. The recommendation  scorer with parameters $\psi$
consumes a fused feature $z_i \coloneqq \big[h_\theta(x_i;\mathcal{G}) \;\Vert\; b_i\big]$,
where $b_i$ are non-graph (auxiliary) features and $\Vert$ denotes concatenation.
The recommendation  score is $s_\psi(z_i)$.
We consider a ranking loss $L_{\mathrm{rec}}(y_i, s_\psi(z_i))$ (pointwise/pairwise/listwise)
and an auxiliary graph loss $L_{\mathrm{gnn}}(\theta)$ (e.g., supervised or SSL).

\subsection*{1.\;Cascaded Baseline (No End-to-End Fine-tuning)}
\textit{Stage A (GNN only):}
\begin{align}
\min_{\theta}\; J_{\mathrm{GNN}}(\theta)=\frac{1}{n}\sum_{i=1}^n L_{\mathrm{gnn}}(\theta).
\end{align}
\textit{Stage B (Recommendation model on frozen embeddings):}
\begin{align}
\min_{\psi}\; J_{\mathrm{rec}}(\psi)=\frac{1}{n}\sum_{i=1}^n
L_{\mathrm{rec}}\bigl(y_i,\, s_\psi(\,[h_{\theta_{\mathrm{frozen}}}(x_i;\mathcal{G}) \Vert b_i]\,)\bigr).
\end{align}
Gradients decouple:
\begin{align}
\nabla_{\theta} J_{\mathrm{rec}}(\psi)=\mathbf{0} \quad\text{and}\quad
\nabla_{\psi} J_{\mathrm{GNN}}(\theta)=\mathbf{0},
\end{align}
so neither module’s optimization can influence the other during training.

\subsection*{2.\;E2E-GRec: Joint Feature Fusion and Loss Fusion}
We jointly optimize
\begin{align}
\min_{\theta,\psi,\alpha_1,\alpha_2}\; J(\theta,\psi,\alpha)
&=\frac{1}{n}\sum_{i=1}^n \Big[
\alpha_1\,L_{\mathrm{rec}}\bigl(y_i,\, s_\psi(z_i)\bigr)
+\alpha_2\,L_{\mathrm{gnn}}(\theta)\Big],\\
z_i &= \big[h_\theta(x_i;\mathcal{G}) \Vert b_i\big],
\label{eq:e2e-objective}
\end{align}
where $\alpha_1,\alpha_2>0$ are task weights learned by GradNorm to balance gradient
magnitudes on a shared layer.

\paragraph{Gradients wrt GNN parameters $\theta$.}
By the chain rule,
\begin{align}
\nabla_{\theta} J(\theta,\psi,\alpha)
&=\frac{1}{n}\sum_{i=1}^n
\alpha_1\,
\underbrace{\frac{\partial L_{\mathrm{recommendation}}\big(y_i, s_\psi(z_i)\big)}{\partial s_\psi(z_i)}}_{\neq 0\ \text{if recommendation error}}
\;\underbrace{\frac{\partial s_\psi(z_i)}{\partial z_i}}_{\text{sensitivity}} \nonumber \\ 
& \;\underbrace{\frac{\partial z_i}{\partial h_\theta(x_i;\mathcal{G})}}_{[I\;\;0]}
\;\underbrace{\frac{\partial h_\theta(x_i;\mathcal{G})}{\partial \theta}}_{\text{GNN backprop}}
\;+\;\alpha_2\,\nabla_{\theta} L_{\mathrm{gnn}}(\theta).
\label{eq:theta-grad}
\end{align}
Hence the \emph{Rec loss} directly backpropagates into the GNN via the nonzero Jacobian
$\partial s_\psi/\partial z_i$ and the fusion map $z_i=[h_\theta \Vert b_i]$.
If the recommendation has any non-degenerate dependence on the $h_\theta$ coordinates
(e.g., the first recommendation layer has nonzero weights on the $h_\theta$ block),
then the first term in \eqref{eq:theta-grad} is nonzero whenever the recommendation signal is nonzero,
proving \textbf{recommendation $\Rightarrow$ GNN} influence.

\noindent \textit{Gradients wrt recommendation parameters $\psi$.}
\begin{align}
\nabla_{\psi} J(\theta,\psi,\alpha)
=\frac{1}{n}\sum_{i=1}^n
\alpha_1\,
\frac{\partial L_{\mathrm{rec}}\big(y_i, s_\psi(z_i)\big)}{\partial s_\psi(z_i)}
\;\frac{\partial s_\psi(z_i)}{\partial \psi}.
\label{eq:psi-grad}
\end{align}
Although $\nabla_{\psi} J$ does not include a direct $\partial L_{\mathrm{gnn}}/\partial \psi$ term,
it depends on $z_i$, which contains $h_\theta(x_i;\mathcal{G})$.
Thus the \emph{numerical value and direction} of $\nabla_{\psi} J$ are functions of $\theta$. So changing $\theta$ changes $z_i$ and hence \emph{changes the recommendation gradient itself}.
This establishes \textbf{GNN $\Rightarrow$ recommendation} influence.

GradNorm adjusts $(\alpha_1,\alpha_2)$ to balance gradient norms on a chosen shared layer $W$:
letting $G_k\!\coloneqq\!\big\|\nabla_{W}\big(\alpha_k L_k\big)\big\|$ for $k\in\{\mathrm{rec},\mathrm{gnn}\}$,
the $\alpha$’s are updated to drive $G_k$ toward a common target. Because
$G_{\mathrm{rec}}$ depends on $(\theta,\psi)$ via $z$ and $s_\psi$, and
$G_{\mathrm{gnn}}$ depends on $\theta$, the \emph{loss weights themselves} become functions
of both modules’ states. This dynamically re-weights $L_{\mathrm{rec}}$ and $L_{\mathrm{gnn}}$
so that each task’s gradient contribution adaptively influences the other.

\section{Proof of Theorem 2}
\label{app:proof2}

\begin{theorem}[SSL vs.\ Cascaded Ranking Head: Objective Misalignment]\label{thm:ssl_cascade_misalignment}
Let $X\in\mathbb{R}^{n\times d}$ be node features and $Z=f_\theta(X)\in\mathbb{R}^{n\times k}$ $(k\!\ge\! d)$
the encoder output. Assume a linear decoder $W_d\in\mathbb{R}^{k\times d}$ with full column rank
$\mathrm{rank}(W_d)=d$. If the SSL reconstruction objective attains zero loss,
\begin{align}
\|X - ZW_d\|_F^2 = 0,
\end{align}
then $\mathrm{col}(X)\subseteq \mathrm{col}(Z)$ (i.e., $Z$ preserves the full feature subspace of $X$).
Consider the cascaded recommendation head with BPR scores $s_i = w^\top z_i$ for some $w\in\mathbb{R}^k$ and
\begin{align}
\mathcal{L}_{\mathrm{BPR}}(Z)=\sum_{(i,j)\in\mathcal{P}} \ell\!\big(w^\top(z_i - z_j)\big),
\end{align}
where $\ell'(\cdot)$ exists. If $\mathrm{col}(X)\nsubseteq \mathrm{span}(w)$, then the two objectives are misaligned:
there exists a nontrivial subspace
\begin{align}
\mathcal{U}\;=\;\mathrm{col}(X)\cap \ker(w^\top)\;\neq\;\{0\}
\end{align}
such that SSL requires $Z$ to preserve information along $\mathcal{U}$, while BPR is invariant to any perturbations $\{\delta z_i\}$ with
$\delta z_i\in \ker(w^\top)=\{v\in\mathbb{R}^k: w^\top v=0\}$ for all $i$ (hence supplies no gradient on $\mathcal{U}$).
Therefore BPR constrains only the one-dimensional subspace $\mathrm{span}(w)$ and leaves the $(k-1)$-dimensional orthogonal subspace
$\ker(w^\top)$ unconstrained, establishing the misalignment.
\end{theorem}

\begin{proof}
(\emph{SSL side}) $\|X-ZW_d\|_F^2=0$ and $\mathrm{rank}(W_d)=d$ imply $X=ZW_d$, hence
$\mathrm{col}(X)\subseteq \mathrm{col}(ZW_d)\subseteq \mathrm{col}(Z)$.

(\emph{BPR side}) For any pair $(i,j)$, by chain rule
\(
\frac{\partial \ell(w^\top(z_i-z_j))}{\partial z_i}=\ell'(w^\top(z_i-z_j))\,w
\)
and
\(
\frac{\partial \ell(w^\top(z_i-z_j))}{\partial z_j}=-\ell'(w^\top(z_i-z_j))\,w.
\)
Summing over pairs yields
$\frac{\partial \mathcal{L}_{\mathrm{BPR}}}{\partial z_i}=\alpha_i\, w\in \mathrm{span}(w)$ for some scalar $\alpha_i$.
Thus for any $v\in\ker(w^\top)$, $\langle \tfrac{\partial \mathcal{L}_{\mathrm{BPR}}}{\partial z_i}, v\rangle=0$, and
$\mathcal{L}_{\mathrm{BPR}}$ is invariant to perturbations $z_i\!\mapsto\! z_i+\delta z_i$ with $\delta z_i\in\ker(w^\top)$.

(\emph{Misalignment}) If $\mathrm{col}(X)\nsubseteq \mathrm{span}(w)$, then
$\mathcal{U}=\mathrm{col}(X)\cap\ker(w^\top)$ is nontrivial. SSL enforces $Z$ to retain information along $\mathcal{U}$
(as $\mathcal{U}\subseteq \mathrm{col}(X)\subseteq \mathrm{col}(Z)$), whereas BPR provides no constraint or gradient
signal on $\mathcal{U}$ since gradients lie in $\mathrm{span}(w)$. Hence the objectives are structurally misaligned.
Hence:
\begin{itemize}
\item \textbf{SSL objective (linear reconstruction):}
\begin{align}
\min_Z \|X - ZW_d\|^2 \quad\Rightarrow\quad \mathrm{col}(X)\subseteq \mathrm{col}(Z),\ \mathrm{rank}(Z)\ge d.
\end{align}
\item \textbf{BPR objective (cascaded ranking):}
\begin{align}
\min_Z \mathcal{L}_{\mathrm{BPR}}(Z)
\Rightarrow
& \min_Z \mathcal{L}_{\mathrm{BPR}}(Z)
\text{only the projection } w^\top Z \text{ is constrained; }\\
&\ker(w^\top)\ \text{is unconstrained/invariant.}
\end{align}
\end{itemize}
\end{proof}

\section{Hyperparameter Search Space}
\label{app:hyper}

The hyperparameter search space is defined as follows:
    \begin{itemize}
        \item Learning rate for GNN and towers: $\{0.01,\ 0.005,\ 0.001\}$.
        \item Dropout rate for GNN: $\{0.1,\ 0.3,\ 0.5,\ 0.7,\ 0.8\}$.
        \item Dropout rate for attention: $\{0.1,\ 0.3,\ 0.5,\ 0.7,\ 0.8\}$.
        \item Weight decay for GNN: $\{0,\ 1\times10^{-3},\ 5\times10^{-3},\ 8\times10^{-3},\ 1\times10^{-4},\ 5\times10^{-4},\ 8\times10^{-4}\}$.
    \end{itemize}

\end{document}